\documentclass[runningheads]{llncs}
\usepackage[T1]{fontenc}
\usepackage{graphicx}
\usepackage{booktabs}
\usepackage[misc]{ifsym}
\usepackage{etoc}

\usepackage[hidelinks]{hyperref}

\usepackage{url}
\usepackage{algorithm}
\usepackage{algorithmicx,algpseudocode}
\usepackage{xstring}

\usepackage{multirow}
\usepackage{algpseudocode}

\usepackage{bm}
\usepackage{amsmath}
\usepackage{amssymb}
\usepackage{dsfont}

\usepackage{tcolorbox}

\tcbset{
  highlight/.style={
    colframe=#1!60!black,
    colback=#1!10,
    boxrule=0.4pt,
    arc=2pt,
    left=2pt,right=2pt,top=2pt,bottom=2pt,
  }
}

\newtheorem{assumption}{Assumption}

\usepackage{newfloat}
\usepackage{listings}

\usepackage[utf8]{inputenc}
\usepackage[T1]{fontenc}
\usepackage{url}
\usepackage{booktabs}
\usepackage{amsfonts}
\usepackage{nicefrac}
\usepackage{microtype}
\usepackage{xcolor}

\usepackage{mathtools}

\newcommand{\sS}{\mathcal{S}}
\newcommand{\sA}{\mathcal{A}}
\newcommand{\sP}{\mathcal{P}}
\newcommand{\sR}{\mathcal{R}}
\newcommand{\sC}{\mathcal{C}}
\newcommand{\sJ}{\mathcal{J}}
\newcommand{\sM}{\mathcal{M}}
\newcommand{\sT}{\mathcal{T}}
\newcommand{\sW}{\mathcal{W}}

\definecolor{backcolour}{rgb}{0.95,0.95,0.92}
\definecolor{codegreen}{rgb}{0,0.6,0}

\definecolor{codegreen}{rgb}{0,0.6,0}
\definecolor{codegray}{rgb}{0.5,0.5,0.5}
\definecolor{codepurple}{rgb}{0.58,0,0.82}
\definecolor{backcolour}{rgb}{0.95,0.95,0.92}
\definecolor{gray1}{HTML}{E6E6E6}

\definecolor{backcolour}{rgb}{0.95,0.95,0.92}
\definecolor{codegreen}{rgb}{0,0.6,0}
\definecolor{flowcolor}{RGB}{0, 102, 204}
\newcommand{\flow}[1]{\textbf{\textcolor{flowcolor}{\texttt{#1}}}}

\definecolor{backcolour}{rgb}{0.95,0.95,0.92}
\definecolor{codegreen}{rgb}{0,0.6,0}
\lstdefinestyle{myStyle}{
    commentstyle=\color{codegreen},
    basicstyle=\ttfamily\footnotesize,
    breakindent=0em,
    breakatwhitespace=true,
    breaklines=true,
    columns=flexible,
    keepspaces=true,
    numbers=left,
    numbersep=5pt,
    showspaces=false,
    showstringspaces=false,
    showtabs=false,
    tabsize=2,
    framexleftmargin=0pt,
    xleftmargin=0pt,
}

\usepackage{longtable}
\usepackage{framed}

\usepackage{bibunits}

\begin{document}

\title{Scheduling That Speaks: An Interpretable Programmatic Reinforcement Learning Framework}

\titlerunning{Programmatic Reinforcement Learning for Scheduling}

\author{Chengpeng Hu\inst{1} \and
Yingqian Zhang\inst{1}  \and
Hendrik Baier\inst{1,2}}

\authorrunning{C. Hu et al.}

\institute{Eindhoven University of Technology, Eindhoven, the Netherlands \and Centrum Wiskunde \& Informatica, Amsterdam, the Netherlands}

\maketitle

\begin{bibunit}[splncs04]
\begin{abstract}
Deep reinforcement learning (DRL) has recently emerged as a promising approach to solve combinatorial optimization problems such as job shop scheduling. However, the policies learned by DRL are typically represented by deep neural networks (DNNs), whose opaque neural architectures and non-interpretable policy decisions can lead to critical trust and usability concerns for human decision makers. In addition, the computational requirements of DNNs can further hinder practical deployment in resource constrained environments. In this work, we propose ProRL, a novel interpretable programmatic reinforcement learning framework that achieves high-performance scheduling with human-readable and editable programmatic policies (i.e., programs). We first introduce a domain-specific language for scheduling (DSL-S) to represent scheduling strategies as structured programs. ProRL then explores the program space defined by DSL-S using local search to identify incomplete programs, which are subsequently completed by learning their parameters via Bayesian optimization. ProRL learns which scheduling heuristic rules to select, and hence, it naturally incorporates existing heuristics already used in industrial scenarios. Experiments on widely used benchmark instances demonstrate the strong performance of ProRL against existing heuristics and DRL baselines. Furthermore, ProRL performs well under strongly constrained computational resources, such as training with only 100 episodes. Our code is available at \url{https://github.com/HcPlu/ProRL}.

\keywords{Programmatic Reinforcement Learning  \and Interpretable Reinforcement Learning \and Scheduling.}
\end{abstract}

\section{Introduction}

Deep reinforcement learning (DRL) has recently shown its promising performance in solving combinatorial optimization problems (COPs) such as job shop scheduling (JSSP)~\cite{zhang2020learning,zhang2024deep},  where learned policies, represented by deep neural networks (DNNs), are able to schedule jobs to machines with optimized makespan (i.e. minimal completion time).
However, despite its capability to learn high quality scheduling solutions,
a critical question emerges: Can we, as human users, truly \textit{understand} the policies encoded by these deep models?

In practice, obtaining high quality solutions is often not the only criterion~\cite{milani2024explainable,zhang2020learning}. In the industrial settings of JSSP, such as manufacturing, factory operators need to understand why certain scheduling decisions are made before they can trust and implement them, especially when they differ from the behavior of traditional methods~\cite{yates2025explainable}. Moreover, current DRL agents often require significant computational resources to
train and run deep neural networks, making them difficult to apply to resource-limited settings such as edge computing~\cite{wang2024minimizing}.
Recent work on programmatic reinforcement learning (PRL)
explores policies represented as lightweight, human-readable programs to improve resource usage and interpretability. Although these approaches have shown promising results in domains such as games~\cite{trivedi2021learning,aleixo2023show,carvalho2024reclaiming} and robotics~\cite{verma2018programmatically,verma2019imitation},
its applicability in solving COPs
is rarely explored.

In this work, we propose ProRL,

a novel programmatic reinforcement learning framework,
for solving job shop scheduling problems.  ProRL determines decisions, i.e. which jobs to assign to which machines, via human-readable programs. Instead of representing policies using deep neural networks, ProRL constructs policies in the form of \textit{programs} that are easier for humans to understand, verify, and edit. Unlike traditional heuristics, schedules generated by ProRL are adaptive to the environment. In addition,  programmatic policies have advantages for resource-restricted deployment, compared to existing DRL approaches to JSSP.

We first introduce DSL-S, a context-free domain-specific language for scheduling,   that defines the perceptions, control flow (e.g., ``if-else'' statements), and actions used to construct the programmatic policies. We construct the perception module with abstract concepts,  such as the ``available machine ratio'',  rather than raw features like the durations and complement time of all jobs,  for better interpretability and compactness. The action component leverages existing heuristic rules (i.e., small programs) that are well established in job shop scheduling.
The control flow is dynamically guided by interpretable linear models.

We then formulate policy optimization as an iterative bilevel optimization problem.

In the outer loop, the program architecture is discovered using a local search (LS) method over the program space while ignoring the values of numerical parameters. In the inner loop, these numerical parameters are learned from the collected trajectories through Bayesian optimization (BO), with the objective of maximizing task returns.

Our contributions are:

\begin{enumerate}
    \item The proposed \textbf{ProRL} is the first   Interpretable Programmatic Reinforcement Learning framework  to solve JSSP with light-weight and human-understandable programs that human users can verify and edit.
    \item  We introduce the context-free language DSL-S to construct policies, which provides a foundation for future research on explainable scheduling.
    \item ProRL incorporates a bilevel optimization that leverages local search for architecture search and Bayesian optimization for parameter learning. In addition, we derive a approximation performance bound and the time complexity of the programmatic policies.
    \item We evaluate ProRL on eight public JSSP benchmarks and validate its performance in resource-limited scenarios. Experimental results demonstrate the outstanding performance of ProRL, compared to traditional heuristics and DRL methods. In large-scale settings, ProRL's performance is even close to or better than that of the constraint programming solver CP-SAT, given 1 hour solving time. Moreover, even with strongly limited computational budgets of only 100 episodes, the proposed ProRL still demonstrates significantly superior performance than heuristics and DRL methods.
\end{enumerate}

\section{Related Work}
\paragraph{\textbf{Programmatic Reinforcement Learning}}
Programmatic reinforcement learning (PRL) constructs policies in the form of programs instead of neural networks (NN)~\cite{choi2005learning}. PRL has shown promising performance and interpretability, for example in games~\cite{verma2019imitation,aleixo2023show,moraes2024searching}, Karel tasks~\cite{trivedi2021learning,carvalho2024reclaiming,lin2024hierarchical,Moraes2025InnateCoder}, and robotics~\cite{verma2018programmatically,qiu2022programmatic}. Recently, \cite{gu2024pi} extended PRL to traffic signal control using Monte Carlo Tree Search and BO. However, this approach is limited to a specific domain and struggles with the enlarged policy search space induced by BO. Local search~\cite{marino2021programmatic,aleixo2023show} is a common method for programs without learnable parameters, as it is intuitive to generate search neighborhoods with production rules of (usually domain-specific) programming languages.~\cite{verma2018programmatically} searched this programmatic space and additionally learned program parameters by using an oracle (a given target neural policy).~\cite{li2025logic} derived programs from human-provided program sketches and tuned their parameters with BO based on a trained deep RL agent.

In this work, we represent policies with programs whose control flow is steered by interpretable linear models based on an abstract state representation with only a few parameters. We propose a novel bilevel optimization that locally searches for program architectures in the programmatic space and learns program parameters without an oracle.

\paragraph{\textbf{Heuristics and DRL approaches for JSSP}}
Priority dispatching rules (PDRs)~\cite{haupt1989survey,sels2012comparison}, such as First in First Out (FIFO), Shortest Processing Time (SPT), Most Operations Remaining (MOR), Most Work Remaining (MWR), and Least Operations Remaining (LOR), are classic heuristic methods that are widely used in modern manufacturing for their simplicity and fast execution. However, PDRs are  simple greedy rules and often struggle to adapt to diverse and complex scheduling scenarios, which limit their ability to produce high  quality schedules.

Recently, deep reinforcement learning (DRL) methods have shown strong performance for solving JSSPs. One promising approach is to model the entire JSSP as a disjunctive graph and construct the solution through graph neural networks trained by RL algorithms~\cite{zhang2020learning,zhang2024deep,remmerden2025offline}. These end-to-end DRL approaches can obtain high quality solutions, adapt to different scenarios, and generalize well to unseen instances.
However, they reply on complex deep neural networks,
making it difficult to directly understand or verify their reasoning processes, which is particularly problematic in industrial settings due to trust concerns and potential risks~\cite{zhang2020learning}. In addition, heavy neural networks
typically require high-performance computing devices such as GPUs, which limits deployment on edge devices. Moreover, many recent high-performing methods, especially improvement-based approaches such as L2S~\cite{zhang2024deep}, require substantial test-time optimization for each new instance.
Instead of such end-to-end DRL approaches to construct heuristics, it is also popular to use RL to select PDRs~\cite{el2006neural,luo2020dynamic,luo2021dynamic,djurasevic2022selection,gui2023dynamic,wu2024deep}, which decides at each step which heuristic rule to use to assign operations to machines.
Constructing action spaces with PDRs helps with interpretability, but the limitations of opaque neural networks mentioned above remain.

In this work, we combine the strong optimization capability of RL methods with the interpretability of PDRs. Specifically, we follow the line of work of selecting PDRs using RL, but instead, propose to represent policies as \emph{human understandable programs} rather than \emph{neural networks}.

\section{Preliminaries}
\paragraph{\textbf{Job Shop Scheduling Problem Formulation}}
A JSS problem consists of a set of jobs $\sJ$ and a set of machines $\sM$.
Each job $J_i \in \sJ$ is defined as a list of operations $\{o_{i1}\rightarrow\cdots\rightarrow o_{ij}\}$, where $o_{ij}~(1\leq j\leq m)$ should operate on a specific machine $j$ with processing time $p_{ij}\in \mathbb{N}$. Each machine can process only one operation at a time.
The goal of the JSSP is to find a feasible schedule, i.e. start times for all operations $\{S_{ij}\}$, such that a given objective, such as \textit{makespan} $C_{max}$,  is minimized. The makespan is defined as the completion time of the last operation: $C_{max} = \max_{i,j}{C_{ij}=\max_{i,j} S_{ij}+p_{ij}}$.

\paragraph{\textbf{Markov Decision Process}}
\label{par:mdp}
We consider the JSSP as a sequential decision-making process, modeled as a Markov decision process (MDP), defined as a tuple $(\sS,\sA,\sR,\sP,\gamma)$~\cite{sutton2018reinforcement}. $\sS$ is the set of states, $\sA$ is the set of actions, $\sR:\sS\times \sA \times \sS \mapsto \mathbb{R}$ is the reward function, $\sP:\sS\times \sA \times \sS \mapsto [0,1]$ is the transition function and $\gamma$ is the discount factor.
The action space consists of a set of PDRs that consist of small programs or rules of thumb: $\sA = \{FIFO,SPT, MOR, MWR,LOR\}$.
The state space is defined by human-understandable concepts abstracted from the raw state. In COPs, states usually describe the information of the problem instances. For JSSP,  we define the state space as $\{LD,AM,AO,JD,ST\}$:
\textit{machine load balance (LD)}, \textit{available machine ratio (AM)}, \textit{available operation ratio (AO)}, \textit{job remaining time balance (JD)}, and \textit{shortest operation remaining time balance (ST)}.
Details are provided in the appendix (Section \ref{app:dsls}).

A policy $\pi$ is a mapping from states to probability distributions over actions, i.e., $\pi(a|s)$ is the probability of selecting a PDR $a$ in state $s$. The goal is to maximize the cumulative reward, $\max_{\pi}\mathbb{E}_{\tau \sim\pi}[\sum_{t=0}^{\infty} \gamma^t \sR(s_t,a_t,s_{t+1})],$ where $\tau \sim\pi$ denotes a trajectory or episode $(s_0,a_0,s_1,a_1,\dots, s_t,a_t,s_{t+1})$ sampled from $\pi$. In the JSSP, the reward is typically sparse and only available at the final step of the episode, defined as the negative makespan, i.e., $r_t = -\mathds{1}_{\{t = |\tau|\}} \cdot C_{\max}$.

\begin{figure}[t]
    \centering
        \fbox{
        \parbox{0.7\textwidth}{
    \begin{eqnarray*}
    \text{Program}~E &:=& h \mid \text{ if } B \text{ then } E_1 \text{ else }     E_2 \\
    \text{Condition}~B &:=& \phi_{\bm{w}}(1,c_0,c_1, \cdots, c_k )> 0\\
    \text{Action}~h &\in& H
    \end{eqnarray*}
    }}
    \caption{Domain-specific language for scheduling (DSL-S).
$H$ denotes the set of PDRs (heuristics). A condition $B$ represents an interpretable linear model characterized by the parameter vector $\bm{w}$ and the concept set $\mathcal{C}=\{c_0,c_1,\cdots,c_k\}$.
    }

    \label{fig:dsl}
\end{figure}

\section{Programmatic RL for Scheduling (ProRL)}
We first propose DSL-S, a domain-specific language to express programmatic policies. We then introduce a bilevel optimization framework that learns an effective policy, represented as a program for selecting scheduling heuristics,  by
jointly searching over program architectures and optimizing program parameters specified in DSL-S.

This approach enables the construction of programmatic policies without relying on a guiding neural policy. Furthermore, we  provide a theoretical performance bound and analyze the inference-time complexity of the programmatic policy.
\subsection{Domain-specific language for scheduling}
To construct policies with programs, one must define a clear syntax. We define a domain-specific language (DSL) for scheduling policies, named DSL-S. Figure~\ref{fig:dsl} shows the context-free grammar of DSL-S, which describes \textit{perception}, \textit{control flow}, and \textit{action} of the programmatic policy.

Perceptions are highly abstracted with human and domain knowledge~\cite{el2006neural}. We define perceptions as intuitive \emph{concepts}, i.e., human-understandable abstractions from the raw problem state.
Hence, we denote the set of all concepts as $\sC$, and define a mapping for each concept $c_k \in \sC$ from the raw state to the concept's value $g_{c_k}:\sS \rightarrow \mathbb{R}$, where $g_{c_k}$ is realized as a program. In learning policies for solving JSSP, concepts correspond directly to the state space of the MDP formulation,  which consists of characteristics of problem instances and parameters in the environment that are important for deriving good solutions.
Based on the literature ~\cite{el2006neural,djurasevic2022selection,gui2023dynamic},
we define five concepts: $\{LD,AM,AO,JD,ST\}$,

representing machine load balance, available machine ratio, available operation ratio, job remaining time balance, and shortest remaining operation time balance, respectively.
For example, $AM$
returns 0.5 if five out of ten machines are available.

Actions are defined by the PDRs $\{FIFO,SPT, MOR, MWR,LOR\}$.
Those heuristics are human-readable and remain widely applied in real-world manufacturing based on experience and proven practices.
Using these existing heuristics, our approach ensures immediate compatibility with current manufacturing processes while offering better interpretability compared to the raw action space that assigns operations directly.

Control flow is handled by \textit{if-else} statements. These statements consist of a \textit{condition-block}, an \textit{if-block}, and an \textit{else-block}. The condition-block compares a parameterized function of human-understandable concepts $\phi_{\bm{w}}(1, c_0,c_1,\cdots,c_k)$ to zero in order to derive its boolean value. In this work, we consider $\phi_{\bm{w}}$ as an interpretable linear model, defined as $\phi_{\bm{w}}(1,c_0,c_1,\cdots,c_k) := \bm{w}\cdot [1, c_0,c_1,\cdots,c_k]^T$. The value of the condition-block controls whether if-block or else-block is executed. Figure~\ref{fig:program_demo1} shows an example program discovered by ProRL based on DSL-S.

\begin{figure*}[t]
    \centering
    \resizebox{\textwidth}{!}{
        \fbox{
        \parbox{0.7\textwidth}{
            \begin{align*}
            &\flow{if}~\left( 1.00 + 0.79 \cdot \bm{LD} - 0.84\cdot \bm{AM} +1.20\cdot \bm{AO} - 0.84\cdot \bm{JD} -1.84\cdot \bm{ST} > 0 \right): \\
             &\qquad \flow{then}~\text{MWR} \\
             &\quad \flow{else if}~\left( -1.11 - 0.24\cdot \bm{LD} + 1.66\cdot \bm{AM} + 1.35\cdot \bm{AO} -1.98\cdot \bm{JD} + 1.46\cdot \bm{ST} > 0 \right): \\
             &~~~\qquad \qquad \flow{then}~\text{LOR} \\
             & \qquad \qquad \flow{else}~\text{SPT}
            \end{align*}
        }
        }
    }

    \caption{A programmatic policy discovered by ProRL. Depending on the state, the policy will choose an action from three heuristics: MWR, LOR and SPT.
    The concepts $\{LD,AM,AO,JD,ST\}$ represent \textit{machine load balance}, \textit{available machine ratio}, \textit{available operation ratio}, \textit{job remaining time balance}, and \textit{shortest operation remaining time balance}, respectively.
    }
    \label{fig:program_demo1}
\end{figure*}

To construct a program based on DSL-S, we start with an initial program $E$ and iteratively expand it. For example, $E$ can be expanded into an action $h$ or an \textit{if-else} statement. An action is a terminal node because it cannot be expanded further. In contrast, an \textit{if-else} statement is a non-terminal node, as it contains incomplete components such as the \textit{if-block} and \textit{else-block}. Only once a program is fully expanded can it be executed.

\subsection{Bilevel optimization for programmatic policies}
We formulate program synthesis under DSL-S as a bilevel optimization problem. A programmatic policy $\pi_{\{\sT, \sW\}}$ is defined by its program architecture $\sT$ with a set of parameters  $\sW$. The architecture $\sT$ is represented as a directed acyclic graph (DAG). In an incomplete program, leaf nodes are either terminal nodes (e.g., actions and conditions) or non-terminal nodes (e.g., unexpanded if-blocks and else-blocks). A valid architecture must be complete, i.e., all leaf nodes must be terminal nodes. Condition nodes are special terminal nodes parameterized by $\bm{w}$. The parameter set $\sW = \{\bm{w}_1,\bm{w}_2,\cdots\}$ consists of the parameters associated with these nodes. Parameters are initialized randomly in new condition nodes.

The action probability $\pi_{\{\sT, \sW\}}(a_t \mid s_t)$ is computed using a top-down traversal of the graph.  All condition nodes are computed iteratively based on the results of previous nodes. The graph is traversed according to the results of conditional expressions, i.e., based on $\mathds{1}{\phi_{\bm{w}}(c_0, c_1, \dots, c_k) > 0}$.
A generalized bilevel optimization of the program with respect to the return of tasks $G$ is denoted as $\max_{\sT}\left[ \max_{\sW} \mathbb{E}_{\pi_{\{\sT, \sW\}}}[G]\right]$.

In JSSP, we can rewrite this equation based on the reward function (i.e., $r_t = -\mathds{1}_{\{t = |\tau|\}} \cdot C_{\max}
$). After rearrangement and ignoring $\gamma$ as a constant, it becomes:
\begin{eqnarray}
    \max_{\sT}\left[ \max_{\sW} \mathbb{E}_{\tau\sim\pi_{\{\sT, \sW\}}}[ (- C_{max})]\right]
    \label{eq:jss_do_obj}
\end{eqnarray}

\begin{figure}
    \centering
    \includegraphics[width=\linewidth, trim={0 1mm 0 0}, clip]{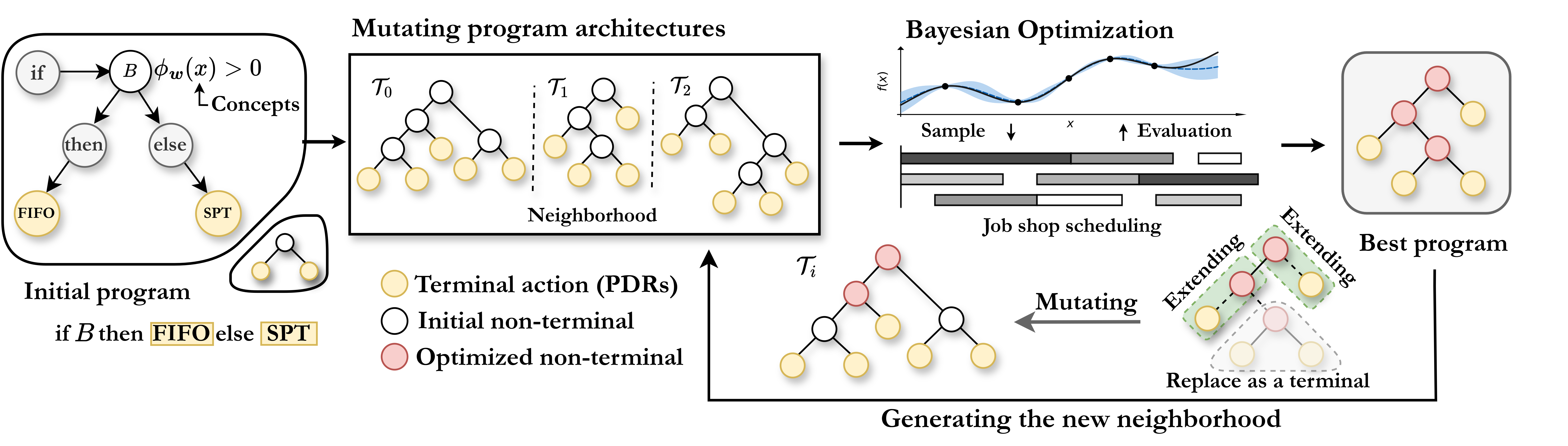}
    \caption{ProRL locally searches for architectures and then optimizes the program parameters via BO. The best program is mutated for generate the new neighborhood.}
    \label{fig:prorl_wf}
\end{figure}
\begin{algorithm}[t]
\caption{ProRL}
\label{alg:lsbo}
\begin{algorithmic}[1]
\Require Number of BO update iterations $\mu$, population size $\lambda$
\Ensure Policy $\pi_{\{\sT,\sW\}}$

\State Initialize $\pi_{\{\sT_1,\sW_1\}}$

\For{$i = 1,2,\cdots$}
    \State $\mathcal{E}^{\lambda}_{\pi_{\{\sT_{i},\sW_{i}\}}}
    = \{\pi_{\{\sT_i^1,\sW_i^1\}},\cdots,\pi_{\{\sT_i^{\lambda},\sW_i^{\lambda}\}}\}$
    \Comment{Generate neighborhood by mutating program architectures}

    \For{$j = 1$ to $\lambda$}
        \Comment{Optimize parameters via BO (architecture $\sT_i^j$ fixed)}

        \State Initialize dataset $D = \emptyset$

        \For{$k = 1$ to $\mu$}
            \State Evaluate policy and obtain return $G_{\sW_i^j}$
            \State $D \leftarrow D \cup \{(\sW_i^j, G_{\sW_i^j})\}$

            \State Update posterior: $p(f \mid D) \propto p(D \mid f)p(f)$
            \State $\sW_i^j \leftarrow
            \arg\max_{\sW} \Lambda(\pi_{\{\sT_i^j,\sW\}})$
        \EndFor
    \EndFor

\State $\mathcal{C} = \{\pi_{\{\sT_i^j,\hat{\sW}_i^j\}}\}_{j=1}^{\lambda} \cup \{\pi_{\{\sT_i,\sW_i\}}\}$
\State $\pi_{\{\sT_{i+1},\sW_{i+1}\}} \leftarrow \arg\max_{\pi \in \mathcal{C}} G_{\pi}$

\EndFor

\end{algorithmic}
\end{algorithm}

\paragraph{\textbf{ProRL with bilevel optimization}}

The programmatic space is usually non-differentiable and highly discontinuous. To practically implement ProRL, we propose an iterative bilevel optimization method (as in Eq.~\eqref{eq:jss_do_obj}), which (a) searches for architectures at the outer level and (b) optimizes the program parameters at the inner level. Algorithm~\ref{alg:lsbo} outlines the procedure for ProRL.

At the outer level, ProRL begins by exploring the programmatic space  to search for policy architectures $\sT$. The neighborhood $\mathcal{E}$ of the local search is defined by mutation operators.
Given a programmatic policy candidate, a mutation operator randomly selects a node, removes its outgoing branches, and expands it using random DSL-S production rules. If the selected node is a terminal node, the operator replaces the selected heuristic with a randomly sampled heuristic. All individuals in the newly generated neighborhood are considered incomplete programs, as their condition parameters have not yet been updated. At the inner level, for each individual, we freeze the architecture and iteratively optimize these parameters via BO. Then, we select the best individual from the neighborhood as the policy for the next generation. The algorithm is terminated according to the episode budget (as counted in BO).

Let $\pi_{\{\sT_i, \sW_i\}}$ denote the policy in generation $i$. ProRL generates a neighborhood $\mathcal{E} = \{\pi_{\{\sT_i^1, \sW_i^1\}},\cdots,\pi_{\{\sT_i^j, \sW_i^j\}}\}$ of size $\lambda$. For each candidate policy $\pi_{\{\sT_i^j, \sW_i^j\}}$ in the neighborhood, ProRL updates the parameters $\sW_i^j$ iteratively using a Bayesian optimization~\cite{shahriari2015taking}.

A Gaussian process (GP) is used as the surrogate model in BO.
We place a prior over the return function $f(\sW)$ (i.e., the unknown reward function). The posterior distribution over
functions is updated via Bayesian update rule
\begin{eqnarray}
p(f \mid D) \propto p(D\mid f) \cdot p(f),
\end{eqnarray}
where $p(f)$ is the GP prior and $p(D \mid f)$ is the likelihood of the observed returns, given the set of observed parameter-performance pairs $D = \{(\sW_i^j, G_{\sW_i^j})_1, \cdots\}$.
The posterior GP provides the predictive distribution $p(G_{\sW_i^j} \mid D)$ for any candidate parameters $\sW_i^j$. We select the best $\sW_i^j$, following $\sW_i^j = \arg \max_{\sW_i^j} \Lambda(\pi_{\{\sT_i^j,\sW_i^j\}})$, where $\Lambda$ is the acquisition function such as upper confidence bound, which quantifies the utility of an observed point~\cite{shahriari2015taking}.

\subsection{Theoretical Analysis}
\label{sec:theory}

\paragraph{\textbf{Approximation Performance Bound}}
For a given policy $\pi$, we define the value function $V^{\pi}(s) \colon = \mathbb{E}_{\pi}[\sum^{\infty}_{k=0}\gamma^kr_{t+k+1}|s_t=s]$.
Consider $\Pi$ as the set of all policies. There exists a stationary and deterministic policy $\pi^*$ such that
    \begin{eqnarray}
        \forall s \in \sS, a\in \sA, V^{\pi^*}(s) = \sup_{\pi\in \Pi}V^{\pi}(s),
    \end{eqnarray}
where $\pi^*$ is an optimal policy and $V^* \colon = V^{\pi^*}$ is the optimal value function~\cite{sutton2018reinforcement}. Similarly, we define the state-action value function and the optimal one as, $Q^{\pi}(s,a) \colon = \mathbb{E}_{\pi}[\sum^{\infty}_{k=0}\gamma^kr_{t+k+1}|s_t=s,a_t=a]$ and $Q^*$.

While RL leverages Bellman equations to optimize the optimal policy, our approach circumvents explicit Bellman updates by directly optimizing over a structured policy class, using local search. The theoretical bound established in Theorem~\ref{theorem2} ensures that with sufficiently expressive programmatic policies (i.e., large enough depth $d$), the optimal value function $V^*$ can be approximated within an error margin that decays exponentially with $d$.
We define $\Pi_{p} \subset \Pi$ as the set of policies that can be expressed by the DSL-S. The optimal \emph{programmatic} policy and its value function are denoted as $\pi^p$ and $V^{\pi^p}$, respectively.

Now we consider the state space is normalized to $[0,1]^n$. The depth of the programmatic policy is $d \in \mathbb{Z}^{+}$.
\begin{assumption}
\label{assump:shrinkage}

Let every internal split node $u$ (i.e., condition) in a programmatic policy partition $\mathbb{S}_u$ into two child partitions $\mathbb{S}_{\text{l}}$ and $\mathbb{S}_{\text{r}}$ such that for a constant $c\in(0,1)$ we have
\begin{eqnarray}
    \max(|\mathbb{S}_{\text{l}}|,|\mathbb{S}_{\text{r}}|)
\le c\cdot |\mathbb{S}_u|
\end{eqnarray}
\end{assumption}
\begin{assumption}
\label{assump:vL}
    $V^*$ and $Q^*$ are Lipschitz continuous, such that there exist constants $L_V$ and $L_Q$, for two arbitrary states $s$ and $s'$, $|V^*(s)-V^*(s')|\leq L_V|s-s'|$ and $|Q^*(s,a)-Q^*(s',a)|\leq L_Q|s-s'|$.

\end{assumption}

We show the uniform exponential approximation bound under a geometric assumption on the partitioning state space via a programmatic policy.

\begin{theorem}
    An optimal programmatic policy $\pi^p$ satisfies:
    \begin{eqnarray}
        |V^*(s)-V^{\pi^p}(s) | \leq \frac{L_V+L_Q}{1-\gamma} \cdot c^{d}
    \end{eqnarray}
    \label{theorem2}
\end{theorem}
\begin{proof}
By Assumption~\ref{assump:shrinkage}, each partition $\mathbb{S}_i$ has diameter at most $c^{d}$. Under Assumption~\ref{assump:vL}, we observe that $|V^*(s)-V^*(s')|\leq L_V\cdot c^{d}$ and $|Q^*(s,a)-Q^*(s',a)|\leq L_Q\cdot c^{d}$. Recall DSL-S, $\pi^p$ satisfies $\forall s\in\mathbb{S}_i,\exists a_i \in \sA, \pi^p(s)=a_i.$ After rearrangements, we obtain the inequality that $V^*(s)-Q^{*}(s,\pi^{p}(s)) \leq  (L_V+L_Q)\cdot c^{d}$. Finally, we derive Theorem~\ref{theorem2}.
    The proof is detailed in the appendix (Section~\ref{app:theory}).
\end{proof}

\paragraph{\textbf{Inference Time Complexity Analysis}}
We analyze the time complexity of ProRL inference and compare it with DRL methods that select PDRs.
For a program with depth $d$, the time complexity to obtain an action is $\mathcal{O}(d \cdot k)$, where $k = |C|$ denotes the size of the concept set. This is derived by viewing the policy as a graph in which each node connects only to its immediate successor. We can directly discard the opposite subprograms of a node based on whether $\phi_{\bm{w}}(c_0, \ldots, c_k) > 0$ holds.
The time complexity of a PDR is usually related to the number of unscheduled operations $n$, i.e., $\mathcal{O}(n\log{n})$; for simplification, we assume the time complexity of a PDR as $\mathcal{O}(\kappa)$. Then the total time complexity for scheduling a single operation becomes $\mathcal{O}(d \cdot k + \kappa)$. This suggests that the total time complexity of the programmatic policy derived from our ProRL is effectively close to that of the PDR, as both $d$ and $k$ are typically bounded. Consider a DRL agent parameterized by a multilayer perceptron (MLP) with $L$ hidden layers, where the $i$-th layer has width $h_i$. This agent takes $k$ concepts as input and selects a PDR from a heuristic set. The time complexity of the DRL agent is $\mathcal{O}\left(\sum^L_{i=0} h_i\cdot h_{i+1}+\kappa \right)$, where $h_0=k$ and $h_{L+1}$ is the size of the heuristic set. Hence, given the common setting where $d$ is small, the programmatic policy can have lower inference overhead than the DRL policy, while considering the same PDR execution cost $\kappa$, i.e, $\mathcal{O}(d \cdot k + \kappa)\leq \mathcal{O}\left(\sum^L_{i=0} h_i\cdot h_{i+1}+ \kappa\right)$, where $h_0=k$.

\section{Experiments}
\label{sec:exp}
We evaluate ProRL on several classic JSSP benchmarks, including TA~\cite{taillard1993benchmarks} and DMU~\cite{demirkol1998benchmarks}. Additional results of other benchmarks such as LA~\cite{lawrence1984resouce} and SWV~\cite{storer1992new} are detailed in the appendix (Section~\ref{app:exp}). The scale of instances is denoted as ``$\#~\text{jobs} \times \#~\text{machine}$''. We select FIFO, SPT, MOR, MWR and LOR as baseline PDRs, since they are also the heuristics of DSL-S. The ``random'' agent randomly selects PDRs. The best result among PDRs is denoted as ``mPDR''. In addition, we evaluate the performance of the OR-Tools CP-SAT solver~\cite{perron2023CPSAT}. Following~\cite{wu2024deep}, we implement a DRL agent that adaptively selects PDR with Proximal Policy Optimization (PPO)~\cite{schulman2017proximal}. ProRL and $\text{PPO}_{\text{PDR}}$ are trained for 10,000 episodes for each instance with three seeds, while the nearly optimal solver, CP-SAT, is given a time limit of one hour, as suggested in the literature~\cite{zhang2020learning,zhang2024deep}. To be fair, we choose the best CP-SAT results from the literature~\cite{zhang2020learning,zhang2024deep}. We evaluate the performance with its \textit{gap} to the best known solutions (BKS)~\footnote{\url{https://optimizizer.com/jobshop.php}}, defined as $gap=\frac{f(C_{max})-f(C_{max})_{\text{BKS}}}{f(C_{max})_{\text{BKS}}}$. We denote the no gap value with ``-''. All \textit{gap} values are averaged over a set of instances with fixed random seeds. Details of the experimental setting can be found in the appendix (Section~\ref{app:hp_es}).

\subsection{Comparing with PDR Heuristics and DRL}

Table~\ref{tab:results_benchmarks_comparision} shows the performance of our ProRL, PDRs and $\text{PPO}_{\text{PDR}}$, considering CP-SAT as a nearly ground-truth solver and BKS on the public JSSP benchmarks including DMU~\cite{demirkol1998benchmarks} and TA~\cite{taillard1993benchmarks}. We also report other results such as LA~\cite{lawrence1984resouce} and SWV~\cite{storer1992new} in the appendix (Section~\ref{app:exp}), as well as details of training time and comparisons with the results reported by~\cite{han2020research,wu2024deep}.
Tab.~\ref{tab:results_benchmarks_comparision} demonstrates that ProRL outperforms all PDRs and $\text{PPO}_{\text{PDR}}$. $\text{PPO}_{\text{PDR}}$ learns to choose PDRs and outperforms those PDRs, but its performance is significantly lower than ProRL.

A key challenge for JSSP is the sparse task return. Recall that the reward function $r_t = -\mathds{1}_{\{t = |\tau|\}} \cdot C_{\max}$, i.e., the negative makespan, is given only at the last step of the episode. This sparse feedback makes it difficult for the value function of the DRL agents to accurately estimate the utility of states. Tab.~\ref{tab:results_benchmarks_comparision} shows that the performance of $\text{PPO}_{\text{PDR}}$ is usually not better than that of PDRs, especially in large-scale instance sets such as ta $100\times20$. ProRL is less affected by the sparse reward. Unlike $\text{PPO}_{\text{PDR}}$ that depends heavily on temporal feedback, ProRL leverages task returns in a Monte Carlo-like way. Both the program architecture search and parameter learning rely solely on the final reward, which naturally mitigates the sparse feedback issue. This advantage is evident in Tab.~\ref{tab:results_benchmarks_comparision}, where ProRL achieves superior performance on large-scale instance sets compared to smaller ones. ProRL achieves competitive performance with CP-SAT in larger-scale instances, for example, 5.58\% gap in TA $50\times15$, 7.13\% in TA $50\times20$, but only 0.97\% in TA $100\times20$. ProRL also presents outstanding performance on benchmarks like LA and SWV, detailed in Section E (Tab.~\ref{tab:results_additional_benchmarks_comparision}) of the appendix.

\begin{table*}[h]

\centering
\setlength{\tabcolsep}{1pt}
\caption{Results (gaps to BKS) grouped by benchmarks~\cite{taillard1993benchmarks,demirkol1998benchmarks}. The best results (including PDRs and $\text{PPO}_{\text{PDR}}$) are bold. $\text{PPO}_{\text{PDR}}$ learns to select PDRs with a neural network. The ``random'' agent randomly chooses PDRs. }

\begin{tabular}{cc|c|cccccc|rr}
\toprule
\multicolumn{2}{c|}{Scale}  & CP-SAT & FIFO & SPT & MOR & MWR & LOR & Random & $\text{PPO}_{\text{PDR}}$ & ProRL \\
\midrule

\multirow{8}{*}{\rotatebox{90}{DMU}}& $20 \times 15$ & 1.80\% & 43.25\% & 28.27\% & 30.26\% & 28.59\% & 44.49\% & 34.67\% & 23.25\% & \textbf{13.40\%} \\
& $20 \times 20$ & 1.90\% & 40.06\% & 31.50\% & 26.88\% & 26.82\% & 42.91\% & 31.52\% & 19.29\% & \textbf{13.30\%} \\
& $30 \times 15$ & 2.50\% & 41.25\% & 31.94\% & 36.40\% & 31.92\% & 45.58\% & 35.44\% & 20.47\% & \textbf{14.61\%} \\
& $30 \times 20$ & 4.40\% & 42.90\% & 35.08\% & 33.70\% & 30.85\% & 48.88\% & 35.80\% & 22.48\% & \textbf{16.18\%} \\
& $40 \times 15$ & 4.10\% & 41.45\% & 23.91\% & 35.52\% & 26.76\% & 41.08\% & 31.03\% & 17.19\% & \textbf{11.01\%} \\
& $50 \times 15$ & 3.80\% & 31.77\% & 24.96\% & 34.64\% & 27.44\% & 32.27\% & 27.28\% & 16.90\% & \textbf{9.34\%} \\
& $40 \times 20$ & 4.60\% & 42.06\% & 37.20\% & 36.03\% & 32.21\% & 45.76\% & 34.92\% & 25.29\% & \textbf{15.16\%} \\
& $50 \times 20$ & 4.80\% & 38.75\% & 30.61\% & 36.10\% & 30.42\% & 42.34\% & 33.40\% & 21.15\% & \textbf{14.36\%} \\
\midrule

\multirow{8}{*}{\rotatebox{90}{TA}}& $15 \times 15$ & 0.02\% & 34.88\% & 25.89\% & 20.53\% & 19.15\% & 40.93\% & 27.91\% & 16.96\% & \textbf{9.14\%} \\
& $20 \times 15$ & 0.20\% & 47.12\% & 32.82\% & 23.55\% & 23.35\% & 51.38\% & 31.36\% & 18.38\% & \textbf{11.42\%} \\
& $20 \times 20$ & 0.70\% & 42.00\% & 27.75\% & 21.71\% & 21.81\% & 40.37\% & 30.00\% & 18.61\% & \textbf{11.32\%} \\
& $30 \times 15$ & 2.10\% & 44.60\% & 35.26\% & 22.82\% & 23.91\% & 55.71\% & 31.53\% & 17.67\% & \textbf{11.29\%} \\
& $30 \times 20$ & 2.80\% & 50.78\% & 34.43\% & 24.93\% & 25.16\% & 57.00\% & 34.03\% & 20.83\% & \textbf{14.87\%} \\
& $50 \times 15$ & 0.00\% & 33.05\% & 24.11\% & 17.37\% & 16.86\% & 35.56\% & 20.33\% & 13.95\% & \textbf{5.81\%} \\
& $50 \times 20$ & 2.80\% & 38.95\% & 25.54\% & 17.68\% & 17.95\% & 43.56\% & 23.89\% & 16.67\% & \textbf{7.08\%} \\
& $100 \times 20$ & 3.90\% & 24.14\% & 14.41\% & 9.15\% & 8.31\% & 30.23\% & 12.60\% & 7.06\% & \textbf{1.02\%} \\
\bottomrule
\end{tabular}

\label{tab:results_benchmarks_comparision}
\end{table*}

\begin{table*}[h]

\centering
\caption{Results of the low-budget performance validation on benchmarks~\cite{taillard1993benchmarks,demirkol1998benchmarks}. ProRL outperforms PDRs and $\text{PPO}_{\text{PDR}}$, even with only 100 episode budget. }
\setlength{\tabcolsep}{1pt}
\begin{tabular}{ll|c|ccc|cccccr}
\toprule
\multicolumn{2}{c|}{\multirow{2}{*}{Scale}} & \multirow{2}{*}{CP-SAT}&
\multirow{2}{*}{mPDR} &
\multirow{2}{*}{Random} &
\multirow{2}{*}{$\text{PPO}_{\text{PDR}}$} &
\multicolumn{5}{c}{ProRL (episode budget)}  \\
 &&  & & & & 0 & 100&200 &1000 & 10000 \\
\midrule

\multirow{8}{*}{\rotatebox{90}{DMU}}& $20 \times 15$ & 1.80\% & 22.84\% & 34.67\% & 23.25\% & 33.17\% & 18.31\% & 17.53\% & 15.78\% & \textbf{13.40\%} \\
& $20 \times 20$ & 1.90\% & 22.32\% & 31.52\% & 19.29\% & 34.66\% & 18.26\% & 17.72\% & 15.16\% & \textbf{13.30\%} \\
& $30 \times 15$ & 2.40\% & 26.84\% & 35.44\% & 20.47\% & 35.99\% & 20.00\% & 19.92\% & 17.03\% & \textbf{14.61\%} \\
& $30 \times 20$ & 4.40\% & 26.38\% & 35.80\% & 22.48\% & 39.59\% & 21.12\% & 21.17\% & 18.02\% & \textbf{16.18\%} \\
& $40 \times 15$ & 4.10\% & 20.57\% & 31.03\% & 17.19\% & 28.95\% & 16.14\% & 16.03\% & 13.33\% & \textbf{11.01\%} \\
& $40 \times 20$ & 4.60\% & 27.88\% & 34.92\% & 25.29\% & 40.06\% & 20.72\% & 20.74\% & 17.78\% & \textbf{15.16\%} \\
& $50 \times 15$ & 3.80\% & 18.24\% & 27.28\% & 16.90\% & 28.63\% & 13.46\% & 12.92\% & 11.09\% & \textbf{9.34\%} \\
& $50 \times 20$ & 4.80\% & 24.06\% & 33.40\% & 21.15\% & 34.18\% & 19.57\% & 19.26\% & 16.49\% & \textbf{14.36\%} \\
\midrule

\multirow{8}{*}{\rotatebox{90}{TA}}& $15 \times 15$ & 0.02\% & 17.71\% & 27.91\% & 16.96\% & 31.12\% & 13.57\% & 13.46\% & 10.88\% & \textbf{9.14\%} \\
& $20 \times 15$ & 0.20\% & 21.54\% & 31.36\% & 18.38\% & 38.75\% & 16.49\% & 16.74\% & 13.53\% & \textbf{11.42\%} \\
& $20 \times 20$ & 0.70\% & 20.36\% & 30.00\% & 18.61\% & 31.92\% & 15.24\% & 15.19\% & 13.08\% & \textbf{11.32\%} \\
& $30 \times 15$ & 2.10\% & 21.52\% & 31.53\% & 17.67\% & 41.22\% & 16.45\% & 16.30\% & 13.70\% & \textbf{11.29\%} \\
& $30 \times 20$ & 2.80\% & 23.44\% & 34.03\% & 20.83\% & 41.63\% & 19.18\% & 19.24\% & 16.52\% & \textbf{14.87\%} \\
& $50 \times 15$ & 0.00\% & 15.48\% & 20.33\% & 13.95\% & 28.71\% & 9.95\% & 9.18\% & 7.24\% & \textbf{5.81\%} \\
& $50 \times 20$ & 2.80\% & 16.57\% & 23.89\% & 16.67\% & 32.14\% & 11.86\% & 11.68\% & 8.42\% & \textbf{7.08\%} \\
& $100 \times 20$ & 3.90\% & 7.74\% & 12.60\% & 7.06\% & 18.89\% & 4.19\% & 4.20\% & 1.77\% & \textbf{1.02\%} \\
\bottomrule
\end{tabular}

\label{tab:low_budget}
\end{table*}

\begin{table}[h]
\centering
\setlength{\tabcolsep}{12pt}
\caption{Average training time of $\text{PPO}_{\text{PDR}}$ and ProRL in seconds.}
\label{tab:avg-train-time-benchmarks}
\begin{tabular}{l|c|ccccc}
\toprule
\multirow{2}{*}{Benchmark} & \multirow{2}{*}{$\text{PPO}_{\text{PDR}}$}  & \multicolumn{5}{c}{ProRL (episode budget)} \\
& & 0 & 100 & 200 & 1000 & 10000 \\
\midrule

DMU & 3993.54 & 0.05 & 50.90 & 51.70 & 167.57 & 1317.42 \\
TA & 4603.02 & 0.06 & 49.98 & 50.62 & 165.51 & 1329.71 \\
\bottomrule
\end{tabular}
\end{table}
\subsection{Performance with Low Computational Budgets}
We evaluate ProRL with limited computational budgets. ProRL is given 0, 100, 200, 1000 and 10,000 episode budgets for training, respectively. $\text{PPO}_{\text{PDR}}$ is still trained with 10,000 episodes. The results are reported in Tab.~\ref{tab:low_budget}. We first observe that ProRL with 0 episodes performs the worst on all benchmarks. This makes sense since ProRL generates an initial policy with a random program architecture and random parameters. However, shown in Tab.~\ref{tab:low_budget}, ProRL with only 100 episodes already outperforms the best PDR and $\text{PPO}_{\text{PDR}}$. As we increase the computation budget, the performance of ProRL improves further. For example, the gap value of ProRL reduces from 14.59\% to 8.82\% in the TA $15\times15$ instance dataset, given budgets from 100 to 10,000 episodes. In larger sets of instances such as $50\times20$, the gap value similarly drops from 10.80\% to 7.13\%. According to Tab.~\ref{tab:avg-train-time-benchmarks}, training time of ProRL is much smaller than $\text{PPO}_{\text{PDR}}$.

ProRL benefits from the simple, yet efficient DSL-S to construct policies with low training cost. The bilevel optimization method supports searching the program architecture and parameters separately, which reduces the policy search space in a hierarchical way. In addition, ProRL's local search treats the best policy of the last generation as the basis of the current generation, partly retaining successful partial architectures and parameters.

In addition, with only 100 episodes, ProRL takes 49.06 s on $100 \times 20$ and achieves a 4.19\% gap to BKS (CP-SAT has a 3.90\% gap), outperforming L2D, CL, and L2S despite their longer training time. This shows the potential of ProRL for industrial applications by providing interpretability and lightweight deployment.

\begin{table*}[htbp]
\centering
\setlength{\tabcolsep}{1pt}
\caption{Comparison with reported results of neural methods including L2D~\cite{zhang2020learning}, L2S~\cite{zhang2024deep}, GM~\cite{corsini2024self} and SI GD~\cite{pirnay2024selfimprovement}. The best PDR per instance is denoted as ``mPDR''.}
\begin{tabular}{cc|c|c|ccccc|rr}
\toprule
\multicolumn{2}{c|}{Scale} & CP-SAT & mPDR & L2D & L2S & GM & GM$_{512}$ & SI GD & $\text{PPO}_{\text{PDR}}$ & ProRL \\
\midrule

\multirow{8}{*}{\rotatebox{90}{DMU}}
& $20 \times 15$ & 1.80\% & 22.84\% & 39.00\% & - & 18.00\% & \textbf{11.30\%} & - & 23.25\% & \textbf{13.40\%} \\
& $20 \times 20$ & 1.90\% & 22.32\% & 37.70\% & - & 19.40\% & \textbf{12.30\%} & - & 19.29\% & \textbf{13.30\%} \\
& $30 \times 15$ & 2.50\% & 26.84\% & 42.00\% & - & 21.80\% & \textbf{14.00\%} & - & 20.47\% & \textbf{14.61\%} \\
& $30 \times 20$ & 4.40\% & 26.38\% & 39.70\% & - & 25.70\% & \textbf{15.80\%} & - & 22.48\% & \textbf{16.18\%} \\
& $40 \times 15$ & 4.10\% & 20.57\% & 35.60\% & - & 17.50\% & \textbf{10.90\%} & - & 17.19\% & \textbf{11.01\%} \\
& $40 \times 20$ & 4.60\% & 27.88\% & 39.60\% & - & 22.20\% & \textbf{14.80\%} & - & 25.29\% & \textbf{15.16\%} \\
& $50 \times 15$ & 3.80\% & 18.24\% & 36.50\% & - & 15.70\% & \textbf{10.60\%} & - & 16.90\% & \textbf{9.34\%} \\
& $50 \times 20$ & 4.80\% & 24.06\% & 39.50\% & - & 22.40\% & \textbf{15.00\%} & - & 21.15\% & \textbf{14.36\%} \\
\midrule

\multirow{8}{*}{\rotatebox{90}{TA}}
& $15 \times 15$ & 0.02\% & 17.71\% & 26.00\% & 9.30\% & 13.80\% & \textbf{6.50\%} & 9.60\% & 16.96\% & \textbf{9.14\%} \\
& $20 \times 15$ & 0.20\% & 21.54\% & 30.00\% & 11.60\% & 15.00\% & \textbf{8.80\%} & 9.90\% & 18.38\% & \textbf{11.42\%} \\
& $20 \times 20$ & 0.70\% & 20.36\% & 31.60\% & 12.40\% & 15.20\% & \textbf{9.00\%} & 11.10\% & 18.61\% & \textbf{11.32\%} \\
& $30 \times 15$ & 2.10\% & 21.52\% & 33.00\% & 14.70\% & 17.10\% & 10.60\% & \textbf{9.50\%} & 17.67\% & \textbf{11.29\%} \\
& $30 \times 20$ & 2.80\% & 23.44\% & 33.60\% & 17.50\% & 18.50\% & \textbf{12.70\%} & 13.80\% & 20.83\% & \textbf{14.87\%} \\
& $50 \times 15$ & 0.00\% & 15.48\% & 22.40\% & 11.00\% & 10.10\% & 4.90\% & \textbf{2.70\%} & 13.95\% & \textbf{5.81\%} \\
& $50 \times 20$ & 2.80\% & 16.57\% & 26.50\% & 13.00\% & 11.60\% & 7.60\% & \textbf{6.70\%} & 16.67\% & \textbf{7.08\%} \\
& $100 \times 20$ & 3.90\% & 7.74\% & 13.60\% & 7.90\% & 5.80\% & 2.10\% & \textbf{1.70\%} & 7.06\% & \textbf{1.02\%} \\
\bottomrule
\end{tabular}

\label{tab:comparision_nco}
\end{table*}

\subsection{Comparing with End-to-end Neural Methods}

Although we focus on selecting PDRs instead of directly constructing or improving solutions, we still compare our method to end-to-end neural methods including:  (1) L2D~\cite{zhang2020learning} that selects the next operation to schedule at each step; (2)  L2S~\cite{zhang2024deep}, a DRL improvement method that iteratively improves a complete solution; (3) GM~\cite{corsini2024self}, a generative neural scheduler that samples, and improves full schedules; (4) SI GD~\cite{pirnay2024selfimprovement}, which retrains the model by repeatedly samples candidate solutions.

We distinguish ProRL from these methods because it does not rely on neural networks and considers both performance and interpretability.
Note that these methods either directly choose the next job or keep improving the current complete solution, which differs from our modeling that chooses suitable heuristics for job/operation assignments. They are all based on deep neural networks and generally require strong computational resources such as GPUs. Improvement methods need a resampling strategy with additional budgets, such as hundreds of improvement steps~\cite{zhang2024deep} for a new instance.
Although they usually obtain high performance, they lack interpretability and are difficult to formally verify or rely on post-hoc explanations via continued execution of the policy, which limits their deployment in the real world.
Comparisons in terms of gaps to BKS, collected from corresponding articles, are presented in Tab.~\ref{tab:comparision_nco}. Our ProRL shows promising performance that is comparable to neural methods. Besides, ProRL provides a directly interpretable reasoning process, i.e., human-understandable programs, and supports feature importance analysis without additional policy execution, since the contributions of features can be inspected directly from the learned program structure. It also has low computational resource requirements.

\subsection{Interpretability}
\label{dis: interpretability}
The programmatic policies derived by ProRL provide a clear and understandable reasoning process through programs; see the example in Figure~\ref{fig:program_demo1}. For the perception part, we first abstract the raw state with a set of human-understandable concepts and further make use of linear models, which themselves are considered white-box and interpretable~\cite{DBLP:journals/csur/DwivediDNSRPQWS23,DBLP:journals/inffus/ArrietaRSBTBGGM20}. The linearity of the learned relationship makes the analysis of feature importance easy. However, there is a trade-off between interpretability and expected performance, controlled by the complexity of a programmatic policy. According to Theorem~\ref{theorem2}, a higher depth $d$ results in stronger policies, while the complexity increases according to $\mathcal{O}(d\cdot k+\kappa)$. In practice, large programs are hard for humans to understand. In our experiments, we set the maximum depth to 4, as suggested in PRL~\cite{carvalho2024reclaiming}, which limits the size of the programmatic policy and maintains the capabilities. The choice of $d$ should vary depending on the complexity of real-world scenarios. One potential direction for future work is to treat the newly discovered programs as an additional option for terminal nodes, thereby enabling the construction of nested programs.

Additionally, we demonstrate the interpretation of the programmatic policy shown in Fig.~\ref{fig:program_demo1} with a large language model (LLM), specifically ChatGPT-4.1. We design a prompt template, detailed in the appendix (Section~\ref{app:llm}), to generate textual explanations of the programmatic policy. We find that an LLM can determine the importance of the features of the program's conditions. We also show a policy verification example in the appendix (Section~\ref{app:pv}).

\section{Conclusion and Future Work}
In this paper, we propose the interpretable programmatic reinforcement learning framework ProRL for scheduling with human-readable programs. We introduce a scheduling DSL, named DSL-S, that leverages understandable concepts and existing heuristics for perception and actions.

We propose an iterative bilevel optimization method that performs local search over the program architectures within the programmatic space defined by DSL-S, while using Bayesian optimization to learn the program parameters. We establish that the expressiveness of the programmatic policy is bounded by its depth $d$, i.e., complexity can be traded off against performance. The inference cost of ProRL programs is low.
Experiments on classic benchmarks demonstrate the outstanding performance of ProRL compared to heuristic and DRL baselines. Low-budget performance validation also underlines the effectiveness of ProRL in resource-limited scenarios, still outperforming PDR heuristics and PPO with a 100-episode budget.
In the future, we will explore how to automatically discover concepts and heuristics beyond handcrafting. Extending ProRL to other COPs like vehicle routing is also a promising research direction.

\paragraph{Generative AI Declaration}
We used generative AI for proofreading and to generate natural-language explanations of programmatic policies in discussion. All scientific content, claims, and results were produced and verified by the authors.

\putbib[main]
\end{bibunit}

\clearpage

\appendix

\begin{bibunit}[splncs04]

\setcounter{footnote}{0}
\setcounter{assumption}{0}
\setcounter{theorem}{0}

\section{Discussion}

We discuss the interpretability of the programmatic policies in ``Interpretability'' and present an example to verify it in ``Policy Verification'' below. Programmatic policies are designed to be human-readable programs and therefore more interpretable than neural policies, as discussed in~\cite{verma2018programmatically,verma2019imitation,liu2023hierarchical,gu2024pi}. Programmatic policies enable users to (i) step through a few conditions and trace the branch for the selected PDR, and  (ii) compute the features' importance. In contrast, it is difficult to explain why a particular heuristic is chosen by neural policies. For example, ``if (1.00 +0.79·LD -0.84·AM +1.20·AO -0.84·JD -1.84·ST >0): then MWR'' in Figure 2 indicates that the rule selects MWR when there are available operations (high AO) and low variability among operation times (low ST). When ST is high, the rule tends to avoid MWR because prioritizing the largest remaining workload can allocate resources to long operations and worsen short-term flow.

We examine the derived programmatic policy via LLM. Our intent is to use LLMs only as optional natural-language explainer of the interpretable program. Nevertheless, it is possible to apply a faithfulness check~\cite{turpin2023language} for LLM explanations, e.g., the output must reference only concepts present in the program and must match the sign (positive/negative) of coefficients. In addition, we acknowledge that human-subject studies are a valuable step in future work.

\section{Details of DSL-S}
\label{app:dsls}
We provide a detailed description of both the concept set and the heuristic set used in DSL-S. We construct the concept set for DSL-S based on the concepts proposed in~\cite{el2006neural}, \cite{djurasevic2022selection}, and \cite{gui2023dynamic}.
\subsection{Concept set}

A JSSP problem consists of a set of jobs $\sJ, $ and a set of machines $\sM$. $n$ and $m$ denote the number of jobs and machines, respectively. We use five concepts: $\{LD,AM,AO,JD,ST\}$ in DSL-S, representing \textit{machine load balance}, \textit{available machine ratio}, \textit{available operation ratio}, \textit{job remaining time balance}, and \textit{shortest operation remaining time balance}, respectively.
\begin{itemize}
    \item Machine load balance: $LD = \frac{\max_j(L_j)-\min_j(L_j)}{\max_J(L_j)}$, where $L_j = \sum_i o_{ij}*\mathds{1}_(o_{ij}=\min_{1\leq k\leq n} o_{kj})$.
    \item Available machine ratio: $AM=\frac{\hat{m}}{m}$, where $\hat{m}$ is the number of available machines.

    \item Available operation ratio: $AO=\frac{|\hat{O}|}{|O|}$, where $\hat{O}$ is the set of available operations.
    \item Remaining time balance: $JD = \frac{\max_i(J_i)-\min_i(J_i)}{\max_i(J_i)}$, where $J_i =\sum_i o_{ij}$,  if $o_{ij}$ is not scheduled.
    \item Shortest operation remaining time balance: $ST=\frac{\max(o_{ij})-\min(o_{ij})}{\max(o_{ij})}$, if $o_{ij}$ is not scheduled,
\end{itemize}
where $1\leq i\leq n$ and $1\leq j \leq m$. $o_{ij}$ denotes the processing time of an operation of job $i$ on the machine $j$.
\subsection{Heuristic set}

We consider First In First Out (FIFO), Shortest Processing Time (SPT), Most Operations Remaining (MOR), Most Work Remaining (MWR) and Least Operations Remaining (LOR) as the heuristic set of DSL-S:
\begin{itemize}
    \item FIFO: Selects the job that arrived earliest is scheduled first.
    \item SPT: Selects the job with the smallest next operation processing time.
    \item MOR: Selects the job with the largest number of remaining operations.
    \item MWR: Selects the job with the largest total remaining processing time.
    \item LOR: Selects the job with the fewest remaining operations.
\end{itemize}
\clearpage
\section{Details of Theoretical Analysis}
\label{app:theory}
\paragraph{\textbf{Approximation Performance Bound}}
For a given policy $\pi$, we define the value function $V^{\pi}(s) \colon = \mathbb{E}_{\pi}[\sum^{\infty}_{k=0}\gamma^kr_{t+k+1}|s_t=s]$.

Consider $\Pi$ as the set of all policies. There exists a stationary and deterministic policy $\pi^*$ such that
    \begin{eqnarray}
        \forall s \in \sS, a\in \sA, V^{\pi^*}(s) = \sup_{\pi\in \Pi}V^{\pi}(s),
    \end{eqnarray}
where $\pi^*$ is called an optimal policy and $V^* \colon = V^{\pi^*}$ is the optimal value function~\cite{sutton2018reinforcement}.

While RL leverages Bellman equations to optimize the optimal policy, our approach circumvents explicit Bellman updates by directly optimizing over a structured policy class, using local search. The theoretical bound established in Theorem~\ref{theorem2} ensures that with sufficiently expressive programmatic policies (i.e., large enough depth $d$), the optimal value function $V^*$ can be approximated within an error margin that decays exponentially with $d$.
We define $\Pi_{p} \subset \Pi$ as the set of policies that can be expressed by the DSL-S. The optimal \emph{programmatic} policy and its value function are denoted as $\pi^p$ and $V^{\pi^p}$, respectively.

\begin{assumption}
\label{assump:shrinkage_app}

Let every internal split node $u$ in a programmatic policy partition its parent cell $\mathbb{S}_u$ into two child cells $\mathbb{S}_{\text{l}}$ and $\mathbb{S}_{\text{r}}$ such that for a constant $c\in(0,1)$ we have
\begin{eqnarray}
    \max(|\mathbb{S}_{\text{l}}|,|\mathbb{S}_{\text{r}}|)
\le c\cdot |\mathbb{S}_u|.
\end{eqnarray}
\end{assumption}

\begin{assumption}
\label{assump:vL_app}
    $V^*$ and $Q^*$ are Lipschitz continuous, such that there exist constants $L_V$ and $L_Q$, for two arbitrary states $s$ and $s'$, $|V^*(s)-V^*(s')|\leq L_V|s-s'|$ and $|Q^*(s,a)-Q^*(s',a)|\leq L_Q|s-s'|$.
\end{assumption}

Now we consider the state space to be normalized to $[0,1]^n$. The depth of the programmatic policy is $d \in \mathbb{Z}^{+}$.
We show that the programmatic policy is bounded.
\begin{theorem}
    A optimal programmatic policy $\pi^p$ satisfies:
    \begin{eqnarray}
        |V^*(s)-V^{\pi^p}(s) | \leq \frac{L_V+L_Q}{1-\gamma} \cdot  c^{d}
    \end{eqnarray}
    \label{theorem2_app}
\end{theorem}

\begin{proof}

    The programmatic policy partitions the state space into $2^{d}$ partitions, whose normalized diameter is at most $ c^{d}$.\\
    Under the assumption~\ref{assump:vL_app}, we obtain the following inequalities within the partition $\mathbb{S}_i$:
    \begin{eqnarray}
       \forall s,s'\in \mathbb{S}_i, |V^*(s)-V^*(s')|\leq L_V\cdot c^{d}
    \end{eqnarray}
        \begin{eqnarray}
       \forall s,s'\in \mathbb{S}_i, |Q^*(s,a)-Q^*(s',a)|\leq L_Q\cdot c^{d}
    \end{eqnarray}

    recall that our DSL describes deterministic policies ($\pi^p$ always chooses the same action within the partition).
    \begin{eqnarray}
        \forall s\in\mathbb{S}_i, \pi^p(s)=a_i,
    \end{eqnarray}
    where $a_i \in \sA$ is a fixed action.

    Let $a_i=\pi^p(s),\forall s\in \mathbb{S}_i$.
    We assume that $a_i=\pi^*(s')$ for some $s'\in\mathbb{S}_i$.

Note that $V^*(s)\ge Q^{*}(s,\pi^{*}(s)) \ge Q^{*}(s,\pi^{p}(s))$
\begin{eqnarray}
V^*(s)-Q^{*}(s,\pi^{p}(s)) &\leq& |V^*(s) -V^*(s')| +|V^*(s')-Q^{*}(s,a_i)| \\
&\leq& |V^*(s) -V^*(s')| +|Q^*(s',\pi^*(s))-Q^{*}(s,\pi^*(s))| \\
&\leq&  (L_V+L_Q)\cdot c^{d}
\end{eqnarray}

Now we consider

\begin{eqnarray}
    V^*(s)-V^{\pi^p}(s) &=&  V^*(s)- Q^{*}(s,\pi^{p}(s)) + Q^{*}(s,\pi^{p}(s)) - V^{\pi^p}(s)\\
    &\leq&  (L_V+L_Q)\cdot c^{d} + Q^{*}(s,\pi^{p}(s)) - Q^{\pi^p}(s,\pi^p(s))\\
    &\leq&  (L_V+L_Q)\cdot c^{d} + \gamma\mathbb{E}_{s'\in P(\cdot|s,\pi^p(s)}[V^*(s')-V^{\pi^p}(s')]\\
        &\leq&  (L_V+L_Q)\cdot c^{d} + \sup_{s\in \sS} \gamma\mathbb{E}_{s'\sim P(\cdot|s,\pi^p(s)}[V^*(s')-V^{\pi^p}(s')]
\end{eqnarray}

Consier that
\begin{eqnarray}
    \Delta := \sup_{s \in \mathcal{S}} \left( V^*(s) - V^{\pi^p}(s) \right) \leq \epsilon + \gamma \Delta \quad \Rightarrow \quad \Delta \leq \frac{\epsilon}{1 - \gamma},
\end{eqnarray}
where $\epsilon=(L_V+L_Q)\cdot c^{d}$, then we derive that

\begin{eqnarray}
    \sup_{s\in \sS}  (V^*(s)-V^{\pi^p}(s)) \leq \frac{L_V+L_Q}{1-\gamma}\cdot c^{d}
\end{eqnarray}
Finally we obtain that
\begin{eqnarray}
    |V^*(s)-V^{\pi^p}(s) | \leq \frac{L_V+L_Q}{1-\gamma} \cdot  c^{d}
\end{eqnarray}
\end{proof}

\clearpage
\newpage
\section{Additional Experimental Results}
\label{app:exp}
We report results on ABZ~\cite{adams1988shifting}, LA~\cite{lawrence1984resouce}, SWV~\cite{storer1992new}, FT~\cite{fish1963ft}, ORB~\cite{applegate1991computational}, and YN~\cite{Yamada1992AGA}. We also provide additional comparisons with the results reported in DQN~\cite{han2020research} and $\text{PPO}_{\text{PDR}}$~\cite{wu2024deep}. Additinoally, we provide the average training time of our ProRL under different budgets and PPO. Finally, a policy verification process via t-SNE~\cite{maaten2008visualizing} is presented.

\subsection{Results on additional instances (ABZ, LA, SWV, FT, ORB, YN)}

Tab.~\ref{tab:results_additional_benchmarks_comparision} shows the comparison of our ProRL with PDRs and DRL. Tab.~\ref{tab:low_budget_additional} shows the performance of ProRL under low-budget conditions and compares ProRL with L2S~\cite{zhang2024deep}. According to the results, ProRL shows promising performance on ABZ, LA, SWV, FT, ORB, and YN benchmarks.

We present the average training time in Tab.~\ref{tab:avg-train-time-benchmarks}. ProRL requires 1322.24 s across all benchmarks, which is much lower than $\text{PPO}_{\text{PDR}}$ (4220.27 s) under the same budget.

\begin{table*}[h]

\centering
\setlength{\tabcolsep}{1pt}
\caption{Additional results grouped by benchmarks ABZ~\cite{adams1988shifting}, LA~\cite{lawrence1984resouce}, SWV~\cite{storer1992new}, FT~\cite{fish1963ft}, ORB~\cite{applegate1991computational} and YN~\cite{Yamada1992AGA}. The best results (including PDR heuristics and PPO) are bold. PPO learns to select PDRs with a neural network. ``Random'' randomly chooses PDRs. ProRL outperforms PDRs and $\text{PPO}_{\text{PDR}}$.}
\begin{tabular}{cc|c|cccccc|rr}
\toprule
\multicolumn{2}{c|}{Scale}  & CP-SAT & FIFO & SPT & MOR & MWR & LOR & Random & $\text{PPO}_{\text{PDR}}$ & ProRL \\
\midrule
\multirow{2}{*}{ABZ}& $10 \times 10$ & 0.00\% & 12.30\% & 12.95\% & 8.80\% & 7.80\% & 23.83\% & 13.65\% & 7.34\% & \textbf{3.56\%} \\
& $20 \times 15$ & 1.15\% & 34.60\% & 33.18\% & 24.06\% & 22.44\% & 44.67\% & 28.65\% & 20.14\% & \textbf{12.00\%} \\
\midrule

\multirow{8}{*}{LA}& $10 \times 5$ & 0.00\% & 17.80\% & 14.81\% & 15.96\% & 16.03\% & 28.90\% & 18.10\% & 5.98\% & \textbf{3.42\%} \\
& $15 \times 5$ & 0.00\% & 7.54\% & 14.86\% & 3.93\% & 4.86\% & 17.45\% & 6.82\% & 2.49\% & \textbf{0.19\%} \\
& $10 \times 10$ & 0.00\% & 25.09\% & 15.67\% & 18.10\% & 12.20\% & 25.69\% & 20.98\% & 9.76\% & \textbf{4.74\%} \\
& $20 \times 5$ & 0.00\% & 7.97\% & 13.72\% & 3.79\% & 4.88\% & 22.21\% & 5.41\% & 2.19\% & \textbf{0.14\%} \\
& $15 \times 10$ & 0.00\% & 29.40\% & 28.69\% & 23.67\% & 17.83\% & 49.86\% & 23.25\% & 15.99\% & \textbf{6.62\%} \\
& $20 \times 10$ & 0.00\% & 24.18\% & 33.43\% & 20.87\% & 17.03\% & 43.03\% & 21.90\% & 19.54\% & \textbf{7.04\%} \\
& $15 \times 15$ & 0.00\% & 22.93\% & 24.59\% & 18.06\% & 17.77\% & 40.30\% & 22.84\% & 14.35\% & \textbf{8.42\%} \\
& $30 \times 10$ & 0.00\% & 11.12\% & 13.89\% & 6.50\% & 8.42\% & 31.56\% & 10.11\% & 5.04\% & \textbf{0.00\%} \\

\midrule
\multirow{3}{*}{SWV}& $20 \times 10$ & 0.10\% & 44.82\% & 26.26\% & 40.49\% & 33.72\% & 34.55\% & 34.57\% & 23.39\% & \textbf{12.87\%} \\
& $20 \times 15$ & 2.50\% & 45.01\% & 32.04\% & 40.86\% & 33.13\% & 41.15\% & 34.34\% & 26.76\% & \textbf{16.38\%} \\
& $50 \times 10$ & 0.00\% & 55.03\% & 21.66\% & 60.15\% & 44.23\% & 29.01\% & 39.12\% & 12.87\% & \textbf{5.94\%} \\

\midrule

\multirow{3}{*}{FT}& $6 \times 6$ & 0.00\% & 9.09\% & 60.00\% & \textbf{7.27\%} & 9.09\% & 23.64\% & 11.52\% & \textbf{7.27\%} & \textbf{7.27\%} \\
& $10 \times 10$ & 0.00\% & 27.31\% & 15.48\% & 25.05\% & 19.14\% & 45.38\% & 26.64\% & 15.48\% & \textbf{8.28\%} \\
& $20 \times 5$ & 0.00\% & 41.20\% & 8.76\% & 37.42\% & 28.84\% & 26.27\% & 30.34\% & 8.76\% & \textbf{3.95\%} \\

\midrule
\multirow{1}{*}{ORB}& $10 \times 10$ & 0.00\% & 29.69\% & 26.30\% & 29.06\% & 24.93\% & 34.90\% & 27.29\% & 19.56\% & \textbf{7.49\%} \\
\midrule
\multirow{1}{*}{YN}& $20 \times 20$ & 0.50\% & 24.64\% & 30.64\% & 22.81\% & 19.71\% & 44.08\% & 21.00\% & 18.26\% & \textbf{10.22\%} \\

\bottomrule
\end{tabular}

\label{tab:results_additional_benchmarks_comparision}
\end{table*}

\begin{table*}[h]

\centering
\caption{Additional results of the low-budget performance validation. ProRL outperforms PDRs and $\text{PPO}_{\text{PDR}}$, even with a budget of only 100 episodes.}
\setlength{\tabcolsep}{1pt}
\begin{tabular}{ll|c|cc|ccccc}
\toprule
\multicolumn{2}{c|}{\multirow{2}{*}{Scale}} & \multirow{2}{*}{CP-SAT}&
\multirow{2}{*}{mPDR} &
\multirow{2}{*}{$\text{PPO}_{\text{PDR}}$} &
\multicolumn{5}{c}{ProRL (episode budget)}  \\
 &&  & & & 0 & 100&200 &1000 & 10000 \\
\midrule

\multirow{2}{*}{ABZ}& $10 \times 10$ & 0.00\% & 6.47\% & 7.34\% & 16.57\% & 4.93\% & 4.32\% & 3.65\% & \textbf{3.56\%} \\
& $20 \times 15$ & 1.15\% & 21.69\% & 20.14\% & 37.48\% & 16.05\% & 16.08\% & 13.18\% & \textbf{12.00\%} \\
\midrule

\multirow{8}{*}{LA}& $10 \times 5$ & 0.00\% & 12.94\% & 5.98\% & 19.50\% & 5.78\% & 5.77\% & 4.07\% & \textbf{3.42\%} \\
& $15 \times 5$ & 0.00\% & 3.23\% & 2.49\% & 15.72\% & 1.37\% & 0.75\% & 0.31\% & \textbf{0.19\%} \\
& $10 \times 10$ & 0.00\% & 9.76\% & 9.76\% & 19.01\% & 8.27\% & 8.46\% & 6.27\% & \textbf{4.74\%} \\
& $20 \times 5$ & 0.00\% & 2.19\% & 2.19\% & 16.55\% & 0.77\% & 0.72\% & 0.40\% & \textbf{0.14\%} \\
& $15 \times 10$ & 0.00\% & 17.14\% & 15.99\% & 35.75\% & 11.45\% & 12.00\% & 8.62\% & \textbf{6.62\%} \\
& $20 \times 10$ & 0.00\% & 15.70\% & 19.54\% & 36.63\% & 12.10\% & 12.60\% & 8.68\% & \textbf{7.04\%} \\
& $15 \times 15$ & 0.00\% & 16.03\% & 14.35\% & 29.82\% & 12.26\% & 11.83\% & 10.16\% & \textbf{8.42\%} \\
& $30 \times 10$ & 0.00\% & 5.32\% & 5.04\% & 19.78\% & 1.12\% & 0.70\% & 0.17\% & \textbf{0.00\%} \\

\midrule

\multirow{3}{*}{SWV}& $20 \times 10$ & 0.10\% & 25.69\% & 23.39\% & 29.02\% & 18.44\% & 19.20\% & 14.91\% & \textbf{12.87\%} \\
& $20 \times 15$ & 2.50\% & 29.38\% & 26.76\% & 35.07\% & 21.48\% & 21.09\% & 18.03\% & \textbf{16.38\%} \\
& $50 \times 10$ & 0.00\% & 21.66\% & 12.87\% & 18.99\% & 8.51\% & 8.05\% & 7.06\% & \textbf{5.94\%} \\
\midrule
\multirow{3}{*}{FT}& $6 \times 6$ & 0.00\% & 7.27\% & 7.27\% & 47.88\% & \textbf{7.27\%} & \textbf{7.27\%} & \textbf{7.27\%} & \textbf{7.27\%} \\
& $10 \times 10$ & 0.00\% & 15.48\% & 15.48\% & 25.45\% & 10.04\% & 13.19\% & 9.25\% & \textbf{8.28\%} \\
& $20 \times 5$ & 0.00\% & 8.76\% & 8.76\% & 14.59\% & 8.58\% & 8.44\% & 8.10\% & \textbf{3.95\%} \\

\midrule
ORB & $10 \times 10$ & 0.00\% & 18.82\% & 19.56\% & 29.17\% & 12.99\% & 13.37\% & 9.00\% & \textbf{7.49\%} \\
\midrule
YN  & $20 \times 20$ & 0.50\% & 17.39\% & 18.26\% & 35.12\% & 13.62\% & 13.45\% & 11.53\% & \textbf{10.22\%} \\

\bottomrule
\end{tabular}

\label{tab:low_budget_additional}
\end{table*}

\begin{table}[h]
\centering
\setlength{\tabcolsep}{12pt}
\caption{Average training time in seconds.}
\label{tab:avg-train-time-benchmarks}
\begin{tabular}{lrrrrrr}
\toprule
\multirow{2}{*}{Benchmark} & \multirow{2}{*}{$\text{PPO}_{\text{PDR}}$}  & \multicolumn{5}{c}{ProRL (episode budget)} \\
& & 0 & 100 & 200 & 1000 & 10000 \\
\midrule
ABZ & 1723.99 & 0.02 & 80.89 & 87.71 & 208.86 & 1279.99 \\
DMU & 3993.54 & 0.05 & 50.90 & 51.70 & 167.57 & 1317.42 \\
TA & 4603.02 & 0.06 & 49.98 & 50.62 & 165.51 & 1329.71 \\
FT & 510.58 & 0.01 & 68.29 & 66.16 & 159.37 & 979.29 \\
LA & 892.72 & 0.01 & 36.19 & 36.10 & 121.39 & 898.74 \\
ORB & 618.69 & 0.01 & 35.26 & 32.91 & 118.09 & 885.78 \\
SWV & 2527.12 & 0.03 & 33.84 & 33.96 & 121.48 & 969.76 \\
YN & 2194.49 & 0.03 & 32.77 & 33.80 & 152.66 & 944.74 \\
\midrule
Overall  & 4220.27 & 0.05 & 51.36 & 52.27 & 167.83 & 1322.24 \\
\bottomrule
\end{tabular}
\end{table}

\clearpage
\subsection{Comparison to Reported Results in the Literature}

We compare ProRL with the results reported in DQN~\cite{han2020research} and $\text{PPO}_{\text{PDR}}$~\cite{wu2024deep}, as presented in Table~\ref{tab:wu2024}.
For the TA~\cite{taillard1993benchmarks} dataset, \cite{wu2024deep} trained their model exclusively on the instances $\text{ta21}~{20 \times 20}$, $\text{ta22}~{20 \times 20}$, $\text{ta31}~{30 \times 15}$, $\text{ta32}~{30 \times 15}$, $\text{ta41}~{30 \times 20}$, $\text{ta42}~{30 \times 20}$, $\text{ta51}~{50 \times 15}$, and $\text{ta52}~{50 \times 15}$.
All results of~\cite{wu2024deep} are directly copied from the published article.
Although $\text{PPO}_{\text{PDR}}$ used 36,000 episodes for training, our ProRL approach achieves superior performance using only 10,000 episodes. These results highlight the effectiveness of ProRL in achieving competitive outcomes with significantly lower computational budgets.
\begin{table*}[h]

\centering
\setlength{\tabcolsep}{0.5pt}
\caption{Comparison to reported results. Results of DQN~\cite{han2020research} and $\text{PPO}_{\text{PDR}}$~\cite{wu2024deep} are directly obtained from their published articles. ProRL outperforms PDRs and DRL baselines.}
\begin{tabular}{cl|l|c|ccccc|ccr}
\toprule

\multicolumn{2}{c|}{\textbf{Instance}}  &\textbf{BKS}&\textbf{CP-SAT} & \textbf{FIFO} & \textbf{SPT} & \textbf{MOR} & \textbf{MWR} & \textbf{LOR} & \textbf{DQN} & \textbf{$\text{PPO}_{\text{PDR}}$} & \textbf{ProRL}   \\
\midrule
\multirow{1}{*}{$\text{ta21}$} & obj. & 1642.0 & 1684.0 & 2208.0 & 2175.0 & 1964.0 & 2044.0 & 2324.0 &1952 & 1876.0& \textbf{1867.7} \\
 ${20 \times 20}$& gap & - & 2.56\% & 34.47\% & 32.46\% & 19.61\% & 24.48\% & 41.53\% &18.88\% & 14.25\% & \textbf{13.74\%} \\
\multirow{1}{*}{$\text{ta22}$} & obj. & 1600.0 & 1639.0 & 2196.0 & 1965.0 & 1905.0 & 1914.0 & 2067.0 &1870 &1794.0 & \textbf{1756.3} \\
${20 \times 20}$ & gap & - & 2.44\% & 37.25\% & 22.81\% & 19.06\% & 19.62\% & 29.19\% &16.88\% & 12.13\% & \textbf{9.77\%} \\

\multirow{1}{*}{$\text{ta31}$} & obj. & 1764.0 & 1786.0 & 2436.0 & 2335.0 & 2143.0 & 2134.0 & 2962.0 &1986 &1965.0& \textbf{1929.0} \\
${30 \times 15}$ & gap & - & 1.25\% & 38.10\% & 32.37\% & 21.49\% & 20.98\% & 67.91\% &12.59\% &11.39\% & \textbf{9.35\%} \\
\multirow{1}{*}{$\text{ta32}$} & obj. & 1784.0 & 1834.0 & 2515.0 & 2432.0 & 2188.0 & 2223.0 & 2923.0 & 2135 &2096.0 & \textbf{2058.7} \\
${30 \times 15}$ & gap & - & 2.80\% & 40.98\% & 36.32\% & 22.65\% & 24.61\% & 63.85\% &19.67\% &17.49\%  & \textbf{15.40\%} \\

\multirow{1}{*}{$\text{ta41}$} & obj. & 2006.0 & 2132.0 & 2973.0 & 2499.0 & 2538.0 & 2620.0 & 2976.0 &2450& 2398.0& \textbf{2348.7} \\
${30 \times 20}$ & gap & - & 6.28\% & 48.21\% & 24.58\% & 26.52\% & 30.61\% & 48.35\% & 22.13\%&19.54\% & \textbf{17.08\%} \\
\multirow{1}{*}{$\text{ta42}$} & obj. & 1939.0 & 2021.0 & 3085.0 & 2710.0 & 2440.0 & 2416.0 & 3445.0 &2351 &2305.0 & \textbf{2202.3} \\
${30 \times 20}$ & gap & - & 4.23\% & 59.10\% & 39.76\% & 25.84\% & 24.60\% & 77.67\% & 22.24\%&18.86\%  & \textbf{13.58\%} \\

\multirow{1}{*}{$\text{ta51}$} & obj. & 2760.0 & 2849.0 & 3717.0 & 3856.0 & 3567.0 & 3435.0 & 3596.0 &3263 &3155.0& \textbf{2978.3} \\
${50 \times 15}$ & gap & - & 3.22\% & 34.67\% & 39.71\% & 29.24\% & 24.46\% & 30.29\% &18.22\% &14.31\% & \textbf{7.91\%} \\
\multirow{1}{*}{$\text{ta52}$} & obj. & 2756.0 & 2830.0 & 3750.0 & 3266.0 & 3303.0 & 3394.0 & 3802.0 &3229 &3056.0& \textbf{2886.0} \\
${50 \times 15}$ & gap & - & 2.69\% & 36.07\% & 18.51\% & 19.85\% & 23.15\% & 37.95\% &17.16\% &10.89\% & \textbf{4.72\%} \\

\bottomrule
\end{tabular}

\label{tab:wu2024}
\end{table*}

\clearpage
\subsection{Policy Verification}
\label{app:pv}
A programmatic policy is usually more verifiable than neural networks~\cite{bastani2018verifiable,verma2019imitation,wang2023verification}. Formal verification can efficiently assess the correctness and robustness of programmatic policies, in domains such as games~\cite{bastani2018verifiable} and continuous control~\cite{verma2019imitation,wang2023verification}. This is especially critical in industrial settings~\cite{dutta2019reachability}, where neural networks are usually unable to provide formal guarantees, leading to concerns about trustworthiness. In contrast, programmatic policies can offer such proofs. We visualize the policy from Fig.~\ref{fig:program_demo1} with t-SNE~\cite{maaten2008visualizing} in Fig.~\ref{fig:tsne} by uniformly sampling 20k state-action pairs. This shows that ProRL can create semantically meaningful decision regions rather than arbitrary conditional branches.
\begin{figure}
    \centering
    \includegraphics[width=0.6\linewidth]{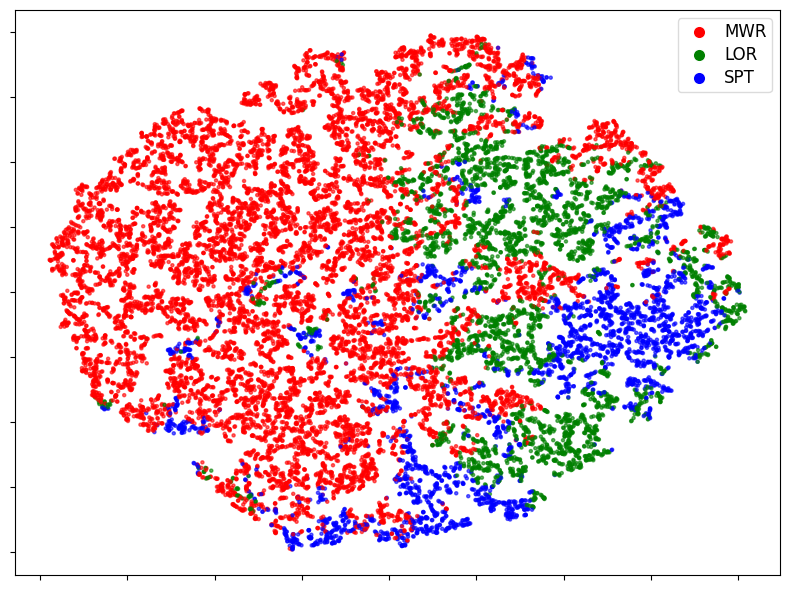}
    \caption{Visualization of the policy shown in Fig.~\ref{fig:program_demo1} via t-SNE. Three PDRs (MWR, LOR and SPT) are clustered well.}
    \label{fig:tsne}
\end{figure}

\clearpage
\subsection{Interpreting ProRL with LLM}
\label{app:llm}
We interpret the programmatic policy shown in Fig.~\ref{fig:program_demo1} with ChatGPT-4.1. Note that~\cite{luo2024end2end} generates textual explanations for neurosymbolic policies (such as formulaic expressions) via LLM. However, providing textual explanations is rarely explored in programmatic reinforcement learning.

Based on the template presented below, it finds the positive/negative influence of different concepts on PDRs:

\begin{lstlisting}[frame=bt,escapeinside={(*@}{@*)},numbers=none,,aboveskip=.2em,belowskip=.1em]
### First Condition (MWR Trigger)
(*@\textbf{AO}@*)((*@\textbf{+1.20}@*)): Strong positive; many available operations encourage MWR.
(*@\textbf{ST}@*)((*@\textbf{-1.84}@*)): Strong negative; large imbalance in short operations discourage MWR
### Second Condition (LOR Trigger)
(*@\textbf{JD}@*)((*@\textbf{-1.98}@*)):: Strong negative; high imbalance in remaining job times discourages LOR.
(*@\textbf{AM}@*)((*@\textbf{+1.66}@*)):: Strong positive; many available machines encourage LOR.
\end{lstlisting}

This highlights the alignment between the condition and heuristic, i.e., MWR (selects the job with the largest total remaining processing time) is chosen when there are many available operations and no urgent need to handle very short jobs first. In this case, MWR prevents long jobs from starving and balances efficiency.

Summarizing the behavior of DRL policies usually requires executing them in the environment, since their mapping from states to actions is opaque. In contrast, we show that LLMs can provide a clear global explanation of a programmatic policy. This description helps users understand the logic and possible consequences of the policy without executing it.

\begin{lstlisting}[frame=bt,escapeinside={(*@}{@*)},numbers=none,,aboveskip=.2em,belowskip=.1em]
(*@\textbf{MWR first}@*): Focus on jobs with the most remaining work when the system has many operations available and resources are balanced.
(*@\textbf{LOR second}@*): If not, prioritize jobs with few operations left, especially when machine availability is high.
(*@\textbf{SPT fallback}@*): When neither condition applies, process the shortest available operation to ensure steady throughput.
\end{lstlisting}

Those results indicate that ProRL could potentially be extended with LLMs for better interpretability, helping users understand and trust the decisions of ProRL.

\paragraph{Prompt template}
The complete prompt template that includes instructions, task description, and background is shown below. To analyze any programmatic policy, one only needs to replace the ``Policy'' section.
\begin{lstlisting}
You need to help a human user understand a programmatic policy for job shop scheduling.

# Task Description
Job Shop Scheduling (JSS) is a type of production scheduling problem where multiple jobs need to be processed on multiple machines, and each job has its own unique sequence of operations.
The goal is to minimize the completion time, or makespan, while respecting machine and job constraints.
A JSS problem consists of a set of jobs $\sJ$ and a set of machines $\sM$.
Each job $J_i \in \sJ$ is defined as a list of operations $\{o_{i1}\rightarrow\cdots\rightarrow o_{ij}\}$, where $o_{ij}~(1\leq j\leq m)$ should operate on a specific machine $j$ with processing time $p_{ij}\in \mathbb{N}$. Each machine can process only one operation at a time.
The goal of the JSSP is to find a feasible schedule, i.e., start times for all operations $\{S_{ij}\}$, such that a given objective, such as the \textit{makespan} $C_{max}$, is minimized. The makespan is defined as the completion time of the last operation: $C_{max} = \max_{i,j}{C_{ij}=\max_{i,j} S_{ij}+p_{ij}}$.

# Background for the Programmatic Policy
The programmatic policy is defined by a domain-specific language for scheduling (DSL-S) as follows:
\begin{eqnarray*}
\text{Program}~E &:=& h \mid \text{ if } B \text{ then } E_1 \text{ else } E_2 \\
\text{Condition}~B &:=& \phi_{\bm{w}}(1,c_0,c_1, \cdots, c_k )> 0\\
\text{Action}~h &\in& H
\end{eqnarray*}

$H$ denotes the set of PDRs (heuristics). A condition $B$ represents an interpretable linear model characterized by the parameter vector $\bm{w}$ and the concept set $\mathcal{C}=\{c_0,c_1,\cdots,c_k\}$.

The heuristic list is defined as $\{FIFO,SPT, MOR, MWR,LOR\}$ using representative PDRs.
\begin{itemize}
    \item FIFO: Selects the job that arrived earliest to be scheduled first.
    \item SPT: Selects the job with the smallest next operation processing time.
    \item MOR: Selects the job with the largest number of remaining operations.
    \item MWR: Selects the job with the largest total remaining processing time.
    \item LOR: Selects the job with the fewest remaining operations.
\end{itemize}

The concept set is defined as $\{LD,AM,AO,JD,ST\}$ in DSL-S, representing \textit{machine load balance}, \textit{available machine ratio}, \textit{available operation ratio}, \textit{job remaining time balance}, and \textit{shortest operation remaining time balance}, respectively.
\begin{itemize}
    \item Machine load balance: $LD = \frac{\max_j(L_j)-\min_j(L_j)}{\max_J(L_j)}$, where $L_j = \sum_i o_{ij}*\mathds{1}_(o_{ij}=\min_{1\leq k\leq n} o_{kj})$.
    \item Available machine ratio: $AM=\frac{\hat{m}}{m}$, where $\hat{m}$ is the number of available machines.
    \item Available operation ratio: $AO=\frac{|\hat{O}|}{|O|}$, where $\hat{O}$ is the set of available operations.
    \item Remaining time balance: $JD = \frac{\max_i(J_i)-\min_i(J_i)}{\max_i(J_i)}$, where $J_i =\sum_i o_{ij}$, if $o_{ij}$ is not scheduled.
    \item Shortest operation remaining time balance: $ST=\frac{\max(o_{ij})-\min(o_{ij})}{\max(o_{ij})}$, if $o_{ij}$ is not scheduled.
\end{itemize}

Here, $1\leq i\leq n$ and $1\leq j \leq m$. $o_{ij}$ denotes the processing time of an operation of job $i$ on machine $j$.

# The Policy
You need to analyze the programmatic policy:
\begin{align*}
&\textbf{{if}}~\bm{(}1.00 + 0.79 \cdot LD - 0.84\cdot AM +1.20\cdot AO - 0.84\cdot JD -1.84\cdot ST >0\bm{)} \\
    &\qquad \text{\textbf{then} MWR} \\
    &\quad \text{\textbf{else: if}}~\bm{(}-1.11 - 0.24\cdot LD + 1.66\cdot AM + 1.35\cdot AO -1.98\cdot JD + 1.46\cdot ST>0\bm{)} \\
    &~~~\qquad \qquad \text{\textbf{then} LOR} \\
    & \qquad \qquad \text{{\textbf{else:}} SPT}
\end{align*}

# Output
You need to analyze (1) global policy behavior, (2) feature importance, and (3) the intuitive relationship between the conditions and the chosen PDRs.
You need to interpret the policy based on your analysis.

Output as a markdown file.
\end{lstlisting}

\clearpage
\section{Hyperparameters and Experimental Settings}
\label{app:hp_es}
The environment and algorithm implementations are adapted from~\cite{reijnen2023job},~\cite{weng2022tianshou}, and~\cite{carvalho2024reclaiming}. We describe the hyperparameters of algorithms and experimental settings.
\subsection{ProRL}
Hyperparameters are listed below:
\begin{itemize}
    \item \# update iteration $\mu$
    \item Maximal depth $d$: 4
    \item The maximal number of tokens: 85
    \item The number of population: 10
    \item Iteration of Bayesian optimization: 20
    \item Acquisition function: Upper confidence Bound
    \item Mutation rate $p_m$: 0.1

    \item The number of initial points: 10
    \item Seeds: 0, 1, 2
    \item 12 workers for parallel environment interaction

\end{itemize}

\subsection{Priority Dispatching Rules}
All PDRs are implemented by~\cite{reijnen2023job}.

\subsection{CP-SAT}

The CP-SAT solver is implemented by~\cite{reijnen2023job} and OR-Tools and follows the default settings of the official example~\footnote{\url{https://github.com/google/or-tools}}. We use Google OR-Tools v9.15 with a random seed of 1 and 16 workers.

Note that we include CP-SAT as a strong solver, as described in the paper ``nearly ground-truth solver.'' We do not claim that our ProRL dominates CP-SAT and all learning-based neural methods, but we aim to validate whether an interpretable programmatic policy can still achieve competitive solution quality. For this reason, CP-SAT is considered as the optimum reference instead of directly involving in comparison. This emphasis on interpretability is essential. CP-SAT can produce excellent schedules, but they do not yield understandable decision rules (``if condition X holds, apply heuristic Y'') that a human can inspect and edit. Neural methods provide decision policies, but they are encoded in high-dimensional networks and are not directly readable. Explaining their decisions typically requires post-hoc analysis that is neither guaranteed nor easily accepted. In contrast, ProRL produces an explicit bounded-depth program in DSL-S, where each decision is attributable to program logic and can be interpreted, verified, and edited by humans.

\subsection{DRL}
Following~\cite{wu2024deep}, we implement a DRL agent, named $\text{PPO}_\text{PDR}$ with proximal policy optimization (PPO)~\cite{schulman2017proximal} based on Tianshou~\cite{weng2022tianshou}.
The DRL agent is trained on the same MDP used by ProRL. Both policy and value network are formed with a $64\times64$ neural network. Hyperparameters are listed below:
\begin{itemize}
    \item $\gamma$: 0.99
    \item batch size: 256
    \item Learning rate: 0.001
    \item GAE lambda: 0.95
    \item Value function coefficient: 0.5
    \item Clipping: 0.2
    \item Number of updates per training iteration: 5
    \item Seeds: 0, 1, 2
    \item 12 workers for parallel environment interaction
\end{itemize}
Other parameters follow the default settings of Tianshou~\cite{weng2022tianshou}.
\subsection{Computational Resource}
All algorithms are trained and tested on a 128-CPU server.
\begin{itemize}
    \item CPU: AMD Rome 7H12 (2x) 64 Cores/Socket 2.6GHz 280W
    \item CPU memory: 256 GiB DRAM (2 GiB per core)
    \item DIMMs: 16 x 16GiB 3200MHz, DDR4
\end{itemize}
For ProRL and DRL, 10,000 episodes are used for training. CP-SAT is given a time limit of 1 hour to solve each instance. Furthermore, ProRL is trained with episode counts of 0, 100, 200, and 1000 to evaluate its performance under limited training conditions.
\subsection{Neural Combinatorial Optimization}
Results of L2D~\cite{zhang2020learning}, L2S~\cite{zhang2024deep} (using 500 improvement steps), GM~\cite{corsini2024self} and SI GD~\cite{pirnay2024selfimprovement} are directly obtain from their published paper.

\clearpage
\section{Details of Benchmarks and Codes}
Tab~\ref{tab:license} provides the links to resources.
\begin{table*}[thb]
\caption{List of licenses for asserts used in this work}
\label{tab:license}
\resizebox{\textwidth}{!}{
\begin{tabular}{lc p{0.2\linewidth} r}
\toprule
\textbf{Resource} & \textbf{Type} & \textbf{Link} & \textbf{License} \\
\midrule
Tianshou~\cite{weng2022tianshou} & Code & \url{https://github.com/thu-ml/tianshou/tree/master?tab=readme-ov-file} &MIT License\\
OR-Tools & Code & \url{https://github.com/google/or-tools} & Apache-2.0 License \\
\cite{carvalho2024reclaiming} & Code & \url{https://github.com/lelis-research/prog_policies} & Available for academic use\\
JSSP environment~\cite{reijnen2023job} &Code\&Dataset & \url{https://github.com/ai-for-decision-making-tue/Job_Shop_Scheduling_Benchmark_Environments_and_Instances} & MIT License\\
TA \cite{taillard1993benchmarks} & Dataset & \url{http://mistic.heig-vd.ch/taillard/problemes.dir/ordonnancement.dir/ordonnancement.html} & Available for academic use \\

Best know solutions & Dataset &\url{https://optimizizer.com/jobshop.php} & Available for academic use\\
\bottomrule
\end{tabular}
}

\end{table*}

\clearpage
\putbib[main]
\end{bibunit}

\end{document}